\documentclass[12pt,draftcls,onecolumn]{IEEEtran}

\usepackage[cmex10]{amsmath}
\usepackage{amssymb}
\usepackage{amsthm}
\usepackage{algorithmic}
\usepackage{array}
\usepackage{fixltx2e}

\usepackage{url}

\newtheorem{theorem}{Theorem}
\newtheorem{assumption}[theorem]{Assumption}
\newtheorem{proposition}[theorem]{Proposition}

\newcommand{\E}{\mathbb{E}}
\newcommand{\R}{\mathbb{R}}

\newcommand{\Dth}{\frac{\partial}{\partial \theta_j}}
\newcommand{\Dt}[1]{\frac{\partial {#1}}{\partial \theta_j}}

\begin{document}
%
\title{Variance Adjusted Actor Critic Algorithms}
%
%
%

\author{Aviv~Tamar,
        Shie~Mannor
\thanks{A. Tamar and S. Mannor are with the Department
of Electrical Engineering, Technion -- Israel Institute of Technology, Haifa,
Israel, 32000.}
}

\markboth{}%
{Shell \MakeLowercase{\textit{et al.}}: Variance Adjusted Actor Critic Algorithms}

\maketitle

\begin{abstract}
We present an actor-critic framework for MDPs where the objective is the variance-adjusted expected return. Our critic uses linear function approximation, and we extend the concept of compatible features to the variance-adjusted setting. We present an episodic actor-critic algorithm and show that it converges almost surely to a locally optimal point of the objective function.
\end{abstract}


\IEEEpeerreviewmaketitle

\section{Introduction}
%
%
%
%
\IEEEPARstart{I}{n} Reinforcement Learning (RL; \cite{BT96}) and planning in Markov Decision Processes (MDPs; \cite{Puterman1994}), the typical objective is to maximize the cumulative (possibly discounted) expected reward, denoted by $J$. When the model's parameters are known, several well-established and efficient optimization algorithms are known. When the model parameters are not known, learning is needed and there are several algorithmic frameworks that solve the learning problem effectively, at least when the model is finite. Among these, actor-critic methods \cite{Konda2002actor} are known to be particularly efficient.

In typical actor-critic algorithms, the critic maintains an estimate of the value function -- the expected reward-to-go. This function is then used by the actor to estimate the gradient of the objective with respect to some policy parameters, and then improve the policy by modifying the parameters in the direction of the gradient. The theory that underlies actor-critic algorithms is the policy gradient theorem \cite{sutton_policy_2000}, which relates the value function with the policy gradient. In practice, for the expected-return objective, actor-critic algorithms have been used successfully in many domains \cite{peters2008natural},\cite{grondman2012survey}.

In many applications such as finance and process control, however, the decision maker is also interested in minimizing some form of \emph{risk} of the policy.
By risk, we mean reward criteria that take into account not only the expected reward, but also some additional statistics of the total reward, such as its variance, denoted by $V$. In this work, we specifically consider a variance-adjusted objective of the form $J - \mu V$, where $\mu$ is a parameter that controls the penalty on the variance.

Recently, several studies have considered RL with such an objective. In \cite{Tamar2012mean_var} a policy gradient (actor only) approach was proposed. Actor-critic algorithms are known to improve over actor-only methods by reducing variance in the gradient estimate, thus motivating the extension of the work in \cite{Tamar2012mean_var} to an actor-critic framework.
In \cite{sato2001td} a variance-penalized actor-critic framework was proposed, but without function approximation for the critic. Function approximation is essential for dealing with large state spaces, as required by any real-world application, and introduces significant algorithmic and theoretical challenges. In this work we address these challenges, and extend the work in \cite{sato2001td} to use linear function approximation for the critic.
In \cite{prasanth2013actor} an actor-critic algorithm that uses function approximation was proposed for the variance-penalized objective. In this algorithm, however, the actor uses simultaneous perturbation methods \cite{spall1992multivariate} to estimate the policy gradient. One drawback of this approach is that convergence can only be guaranteed to locally optimal points of a \emph{modified} objective function, which takes into account the error induced by function approximation. This error depends on the choice of features, and in general, there is no guarantee that it will be small. Another drawback of the method in \cite{spall1992multivariate} is that two trajectories are needed to estimate the policy gradient. In this work, we avoid both of these drawbacks. By extending the policy gradient theorem and the concept of \emph{compatible features} \cite{sutton_policy_2000}, we are able to guarantee convergence to a local optima of the \emph{true} objective function, and require only a single trajectory for each gradient estimate.
Our approach builds upon recently proposed policy evaluation algorithms \cite{Tamar2013var_TD} that learn both the expected reward-to-go and its second moment. We extend the policy gradient theorem to relate these functions with the policy gradient for the variance-penalized objective, and propose an episodic actor-critic algorithm that uses this gradient for policy improvement. We finally show that under suitable conditions, our algorithm converges almost surely to a local optimum of the variance-penalized objective.

\section{Framework and Background}\label{sec:background}

We consider an episodic MDP (also known as a stochastic shortest path problem; \cite{Ber2012DynamicProgramming})
in discrete time with a finite state space $X$, an initial state $x_0$, a terminal state $x^*$, and a finite action space $U$. The transition probabilities are denoted by $P(x'|x,u)$.
We let $\pi_\theta$ denote a policy parameterized by $\theta \in \R^n$, that determines, for each $x\in X$, a distribution over actions $P_\theta(u|x)$. We consider a deterministic and bounded reward function $r:X \to \R$, and assume zero reward at the terminal state.
We denote by $x_k$, $u_k$, and $r_k$ the state, action, and reward, respectively, at time $k$, where $k = 0,1,2,\ldots$.

A policy is said to be \emph{proper} \cite{Ber2012DynamicProgramming} if there is a positive probability that the terminal state $x^*$ will be reached after at most $n$ transitions, from any initial state. Throughout this paper we make the following two assumptions
\begin{assumption}\label{assumption:proper_policy}
The policy $\pi_\theta$ is proper for all $\theta$.
\end{assumption}

\begin{assumption}\label{assumption:grad_log_defined}
For all $\theta \in \R^n$, $x\in X$, and $u\in U$, the gradient $\Dt{ \log P_\theta(u|x)}$ is well defined and bounded.
\end{assumption}
Assumption \ref{assumption:grad_log_defined} is standard in policy gradient literature, and a popular policy representation that satisfies it is softmax action selection \cite{sutton_policy_2000},\cite{MarTsi98}.

Let $\tau\triangleq\min\{k>0|x_{k}=x^*\}$ denote the first visit time to the terminal state, and let the random variable $B$ denote the accumulated (and possibly discounted) reward along the trajectory until that time
\begin{equation*}
B\triangleq\sum_{k=0}^{\tau-1} \gamma^k r(x_k).
\end{equation*}
For a policy $\pi_\theta$ the expected reward-to-go $J^\theta:X \to \R$, also known as the \emph{value function}, is given by
\begin{equation*}
J^\theta(x)\triangleq\E^\theta\left[B | x_{0}=x\right],
\end{equation*}
where $\E^\theta$ denotes an expectation when following policy $\pi_\theta$. We similarly define
the \emph{variance of the reward-to-go} $V^\theta:X \to \R$ by
\begin{equation*}
V^\theta(x)\triangleq\textrm{Var}^\theta\left[B | x_{0}=x\right],
\end{equation*}
and the \emph{second moment of the reward-to-go} $M^\theta:X \to \R$ by
\begin{equation*}
M^\theta(x)\triangleq \E^\theta \left[B^2 | x_{0}=x\right].
\end{equation*}

Slightly abusing notation, we also define corresponding state-action functions\footnote{$J^\theta (x,u)$ is often referred to as the Q-value function and denoted $Q^\theta (x,u)$. Here, we avoid introducing a new notation to the state-action second moment function and use the same notation as for the state-dependent functions. Any ambiguity may be resolved from context.} $J^\theta:X\times U \to \R$, $V^\theta:X\times U \to \R$, and $M^\theta:X\times U \to \R$ by
\begin{equation*}
\begin{split}
J^\theta(x,u) &\triangleq \E^\theta\left[B | x_{0}=x,u_{0}=u\right], \\
V^\theta(x,u) &\triangleq \textrm{Var}^\theta\left[B | x_{0}=x,u_{0}=u\right], \\
M^\theta(x,u) &\triangleq \E^\theta\left[B^2 | x_{0}=x,u_{0}=u\right].
\end{split}
\end{equation*}

Our goal is to find a parameter $\theta$ that optimizes the variance-adjusted expected long-term return
\begin{equation}\label{eq:eta}
\begin{split}
\eta(\theta) = \eta_J(\theta) - \mu \eta_V(\theta) &\triangleq \E^{\theta}\left[B \right] - \mu \textrm{Var}^{\theta}\left[B \right] \\
&\equiv J^\theta(x_0) - \mu V^\theta(x_0),
\end{split}
\end{equation}
where $\mu \in \R$ controls the penalty on the variance of the return. Actor-critic algorithms (which are an efficient variant of policy gradient algorithms) use sampling to estimate the gradient of the objective $\Dth \eta(\theta)$, and use it to perform stochastic gradient ascent on the parameter $\theta$, thereby reaching a \emph{locally} optimal solution. Traditionally, the actor-critic framework has been developed for optimizing the expected return; in this paper we extend it to the variance-penalized setting.

\section{A Policy Gradient Theorem for the Variance}\label{sec:experiments}

Classic actor-critic algorithms are driven by the \emph{policy gradient theorem} \cite{sutton_policy_2000, Konda2002actor}, which states a relationship between the gradient $\Dth \eta_J(\theta)$ and the value function $J^{\theta}(x,u)$. Algorithmically, this suggests a natural dichotomy where the critic part of the algorithm is concerned with learning $J^{\theta}(x,u)$, and the actor part evaluates $\Dth \eta_J(\theta)$ and uses it to modify the policy parameters $\theta$. In this section we extend the policy gradient theorem to performance criteria that include the variance of the long-term return, such as \eqref{eq:eta}.

We begin by stating the classic policy gradient theorem \cite{sutton_policy_2000,Konda2002actor}, given by
\begin{equation}\label{eq:PG_theorem}
\Dt {\eta_J(\theta)} = \Dt{ J^\theta(x_0) }= \sum_{t=0}^{\infty} \sum_{x\in X} P \left( x_t = x|\pi_\theta \right) \delta J^\theta(x),
\end{equation}
where
\begin{equation*}
\delta J^\theta(x) \triangleq \sum_{u\in U} \Dt{ \pi_\theta (u|x) }J^{\theta}(x,u).
\end{equation*}
Note that if $J^{\theta}(x,u)$ is known, estimation of $\Dt{ \eta_J(\theta)}$ from a sample trajectory (with a fixed $\theta$) is straightforward, since \eqref{eq:PG_theorem} may be equivalently written as
\begin{equation}\label{eq:PG_theorem2}
\Dt {\eta_J(\theta)} = \E^{\theta} \left[ \sum_{t=0}^{\infty} \Dt {\log \pi_\theta (u_t|x_t)} J^\theta(x_t,u_t) \right],
\end{equation}
where the expectation is over trajectories.

In \cite{sato2001td}, the policy gradient theorem was extended to the variance-penalized criterion $\eta(\theta)$, using the state-action variance function $V^{\theta}(x,u)$, and without function approximation. Here, we follow a similar approach, and provide an extension to the theorem that uses $J^{\theta}(x,u)$ and $M^{\theta}(x,u)$. Incorporating function approximation will then follow naturally, using the methods of \cite{Tamar2013var_TD}.

We begin by using the relation $V = M - J^2$ to write the gradient of the variance
\begin{equation*}
\Dt {\eta_V(\theta)} = \Dt {M^\theta(x_0)} - 2 J^\theta(x_0) \Dt {J^\theta(x_0)}.
\end{equation*}
In the next proposition we derive expressions for the two terms above in the form of expectations over trajectories. The proof is given in Appendix \ref{supp:prop:PG_var_theorem}.
\begin{proposition}\label{prop:PG_var_theorem}
Let Assumptions \ref{assumption:proper_policy} and \ref{assumption:grad_log_defined} hold. Then
\begin{equation*}
 J^\theta(x_0) \Dt {J^\theta(x_0)} = \E^{\theta} \left[ J^\theta(x_0) \sum_{t=0}^{\infty} \Dt {\log \pi_\theta (u_t|x_t)} J^\theta(x_t,u_t) \right],
\end{equation*}
and
\begin{equation*}
\begin{split}
 \Dt {M^\theta(x_0)} &= \E^{\theta} \left[ \sum_{t=0}^{\infty} \Dt {\log \pi_\theta (u_t|x_t)} M^\theta(x_t,u_t) \right] \\
& + 2\E^{\theta} \left[ \sum_{t=1}^{\infty} \Dt{ \log \pi_\theta (u_t|x_t)} J^\theta(x_t,u_t) \sum_{s=0}^{t-1} r(x_s) \right].
\end{split}
\end{equation*}
\end{proposition}

Proposition \ref{prop:PG_var_theorem} together with \eqref{eq:PG_theorem2} suggest that given $J^\theta(x,u)$ and $M^\theta(x,u)$, a sample trajectory following the policy $\pi_\theta$ may be used to update the policy parameter $\theta$ in the (expected) gradient direction $\Dt {\eta(\theta)}$; this is referred to as the \emph{actor update}. In general, however, $J^\theta$ and $M^\theta$ are not known, and have to be estimated; this is referred to as the \emph{critic update}, and will be performed using the methods of \cite{Tamar2013var_TD}, as we describe next.

\section{Approximation of $J^\theta$ and $M^\theta$, and Compatible Features}\label{sec:experiments}

When the state space $X$ is large, a direct computation of $J^\theta$ and $M^\theta$ is not feasible. For the case of the value function $J^\theta$, a popular approach in this case is to approximate $J^\theta$ by restricting it to a lower dimensional subspace, and use simulation-based \emph{learning} algorithms to adjust the approximation parameters \cite{Ber2012DynamicProgramming}. Recently, this technique has been extended to the approximation of $M^\theta$ as well \cite{Tamar2013var_TD}, an approach that we similarly pursue. One problem with using an approximate $J^\theta$ and $M^\theta$ in the policy gradient formulae of Proposition \ref{prop:PG_var_theorem} is that it biases the gradient estimate, due to the approximation error of $J^\theta$ and $M^\theta$. For the case of the expected return, this issue may be avoided by representing $J^\theta$ using \emph{compatible features} \cite{sutton_policy_2000},\cite{Konda2002actor}. Interestingly, as we show here, this approach may be applied to the variance-adjusted case as well.

\subsection{A Linear Function Approximation Architecture}
We begin by defining our approximation scheme. Let $\tilde{J}^\theta(x,u)$ and $\tilde{M}^\theta(x,u)$ denote the approximations of $J^\theta(x,u)$ and $M^\theta(x,u)$, respectively. For some parameter vectors $w_J\in\R^{l}$ and $w_M\in\R^{m}$ we consider a linear approximation architecture of the form
\begin{equation*}
    \begin{split}
      \tilde{J}^\theta(x,u;w_J) &= \phi_J^\theta(x,u)^{\top} w_J, \\
      \tilde{M}^\theta(x,u;w_M) &= \phi_M^\theta(x,u)^{\top} w_M,
    \end{split}
\end{equation*}
where $\phi_J^\theta(x,u)\in\R^{l}$ and $\phi_M^\theta(x,u)\in\R^{m}$ are state-action dependent features, that may also depend on $\theta$. The low dimensional subspaces are therefore
\begin{equation*}
    \begin{split}
       S_J^\theta &= \{ \Phi_J^\theta w | w\in\R^{l} \}, \\
       S_M^\theta &= \{ \Phi_M^\theta w | w\in\R^{m} \},
     \end{split}
\end{equation*}
where $\Phi_J^\theta$ and $\Phi_M^\theta$ are matrices whose rows are ${\phi_J^\theta}^{\top}$ and ${\phi_M^\theta}^{\top}$, respectively. We make the following standard independence assumption on the features
\begin{assumption}\label{assumption:Phi_rank}
The matrix $\Phi_J^\theta$ has rank $l$ and the matrix $\Phi_M^\theta$ has rank $m$ for all $\theta\in\R^n$.
\end{assumption}
Assumption \ref{assumption:Phi_rank} is easily satisfied, for example, in the case of compatible features and the softmax action selection rule of \cite{sutton_policy_2000}.

We proceed to define how the approximation weights $w_J\in\R^{l}$ and $w_M\in\R^{m}$ are chosen. For a trajectory $x_0,\dots,x_{\tau-1}$, where the states evolve according to the MDP with policy $\pi_\theta$, define the state-action occupancy probabilities
\begin{equation*}
    q_t^{\theta}(x,u) = P(x_t = x, u_t = u|\pi_\theta),
\end{equation*}
and let
\begin{equation*}
    q^\theta(x,u) = \sum_{t=0}^{\infty} q^\theta_t(x,u).
\end{equation*}
We make the following standard assumption on the policy $\pi_\theta$ and initial state $x_0$.
\begin{assumption}\label{assumption:all_states_visited}
For all $\theta\in \R^n$, each state-action pair has a positive probability of being visited, namely, $q^\theta(x,u)>0$ for all $x\in X$ and $u \in U$.
\end{assumption}

For vectors in $\R^{X \times U}$, let $\|y\|_{q^\theta}$ denote the $q^\theta$-weighted Euclidean norm. Also, let $\Pi^\theta_J$ and $\Pi^\theta_M$ denote the projection operators from $\R^{X \times U}$ onto the subspaces $S_J^\theta$ and $S_M^\theta$, respectively, with respect to this norm. The approximations $\tilde{J}^\theta(x,u)$ and $\tilde{M}^\theta(x,u)$ are finally given by
\begin{equation*}
\tilde{J}^\theta = \Pi^\theta_J J^\theta, \quad \text{and} \quad \tilde{M}^\theta = \Pi^\theta_M M^\theta.
\end{equation*}

\subsection{Compatible Features}
The idea of compatible features is to identify features for which the function approximation does not bias the gradient estimation. This approach is driven by the insight that the policy gradient theorem may be written in the form of an inner product as follows \cite{Konda2002actor}. Let $\left< \cdot, \cdot \right>_{q^\theta}$ denote the ${q^\theta}-$weighted inner product on $\R^{X \times U}$:
\begin{equation*}
\left< J_1, J_2 \right>_{q^\theta} \triangleq \sum_{x\in X, u\in U} q^{\theta}(x,u)J_1(x,u)J_2(x,u).
\end{equation*}
Eq. \eqref{eq:PG_theorem} may be written as
\begin{equation*}
    \Dt {\eta_J} = \left< \psi^\theta_j , J^\theta \right>_{q^\theta},
\end{equation*}
where $\psi^\theta_j(x,u) = \Dth \log \pi_\theta (u|x)$. Now, observe that if $Span\left\{\psi^\theta\right\} \subset S_J^\theta$ we have that $\left< \psi^\theta_j , J^\theta \right>_{q^\theta} = \left< \psi^\theta_j , \Pi^\theta_J J^\theta \right>_{q^\theta}$ for all $j$, therefore, replacing the value function with its approximation does not bias the gradient. We now extend this idea to the variance-adjusted case.

We would like to write the gradient of the variance in a similar inner product form as described above. A comparison of the terms in Proposition \ref{prop:PG_var_theorem} with the terms in Eq. \eqref{eq:PG_theorem} shows that the only difficulty is in the second term of $\Dt {M^\theta(x_0)}$, where the sum $\sum_{s=0}^{t-1} r(x_s)$ appears in the expectation. We therefore define the \emph{weighted} state-action occupancy probabilities
\begin{equation*}
    \tilde{q}^\theta(x,u) = \sum_{t=1}^{\infty} P(x_t = x, u_t = u|\pi_\theta) \E^{\theta}\left[ \left. \sum_{s=0}^{t-1} r(x_s)\right| x_t = x, \pi \right],
\end{equation*}
and we make the following positiveness assumption on $\tilde{q}^\theta(x,u)$:
\begin{assumption}\label{assumption:tilde_q_positive}
For all $x,$ $u,$ and $\theta$ we have that $\tilde{q}^\theta(x,u)>0$.
\end{assumption}
Assumption \ref{assumption:tilde_q_positive} may easily be satisfied by adding a constant baseline to the reward. Let $\left< \cdot, \cdot \right>_{\tilde{q}^\theta}$ denote the ${\tilde{q}^\theta}-$weighted inner product on $\R^{X \times U}$, and let $\|y\|_{\tilde{q}^\theta}$ denote the corresponding $\tilde{q}^\theta$-weighted Euclidean norm, which is well-defined due to Assumption \ref{assumption:tilde_q_positive}. Also, let $\tilde{\Pi}^\theta_J$ denote the projection operator from $\R^{X \times U}$ onto the subspace $S_J^\theta$ with respect to this norm.

As outlined earlier, we make the following \emph{compatibility} assumption on the features:
\begin{assumption}\label{assumption:compatible_features}
For all $\theta$ we have $Span\left\{\psi^\theta\right\} \subset S_J^\theta$ and $Span\left\{\psi^\theta\right\} \subset S_M^\theta$.
\end{assumption}

The next proposition shows that when using compatible features, the approximation error does not bias the gradient estimation.
\begin{proposition}\label{prop:compatible}
Let Assumptions \ref{assumption:proper_policy}, \ref{assumption:grad_log_defined}, \ref{assumption:Phi_rank}, \ref{assumption:all_states_visited}, \ref{assumption:tilde_q_positive}, and \ref{assumption:compatible_features} hold. Then
    \begin{equation*}
    \begin{split}
    \Dt {\eta_V} &= \left< \psi_j , M \right>_{q^\theta} + 2\left< \psi_j , J \right>_{\tilde{q}^\theta} - 2J(x_0)\left< \psi_j , J \right>_{q^\theta} \\
    &= \left< \psi_j , \Pi^\theta_M M \right>_{q^\theta} + 2\left< \psi_j , \tilde{\Pi}^\theta_J J \right>_{\tilde{q}^\theta} - 2J(x_0)\left< \psi_j , \Pi^\theta_J J \right>_{q^\theta}
    \end{split}
    \end{equation*}
\end{proposition}

\begin{proof}
By Assumptions \ref{assumption:all_states_visited} and \ref{assumption:tilde_q_positive} the inner products $\left< \cdot , \cdot \right>_{\tilde{q}^\theta}$ and $\left< \cdot , \cdot \right>_{q^\theta}$ are well defined. By Assumptions \ref{assumption:proper_policy} and \ref{assumption:grad_log_defined} Proposition \ref{prop:PG_var_theorem} holds, therefore we have, by definition,
\begin{equation*}
    \Dt {\eta_V} = \left< \psi_j , M \right>_{q^\theta} + 2\left< \psi_j , J \right>_{\tilde{q}^\theta} - 2J(x_0)\left< \psi_j , J \right>_{q^\theta}.
\end{equation*}
By Assumptions \ref{assumption:Phi_rank}, \ref{assumption:all_states_visited}, and \ref{assumption:tilde_q_positive}  the projections $\Pi^\theta_J$, $\Pi^\theta_M$, and $\tilde{\Pi}^\theta_J$ are well-defined and unique. Now, by Assumption \ref{assumption:compatible_features}, and the definition of the inner product and projection we have
\begin{equation*}
\begin{split}
\left< \psi_j , J \right>_{q^\theta} &= \left< \psi_j , \Pi^\theta_J J \right>_{q^\theta}, \\
\left< \psi_j , M \right>_{q^\theta} &= \left< \psi_j , \Pi^\theta_M M \right>_{q^\theta}, \\
\left< \psi_j , J \right>_{\tilde{q}^\theta} &= \left< \psi_j , \tilde{\Pi}^\theta_J J \right>_{\tilde{q}^\theta},
\end{split}
\end{equation*}
yielding the desired result.
\end{proof}

In the next section, based on Proposition \ref{prop:compatible}, we derive an actor-critic algorithm.

\section{An Episodic Actor-Critic Algorithm}

In this section, based on the results established earlier, we propose an episodic actor-critic algorithm, and show that it converges to a locally optimal point of $\eta(\theta)$.

Our algorithm works in episodes, where in each episode we simulate a trajectory of the MDP with a fixed policy. Let $\tau^i$ and $x^i_0,u^i_0,\dots,x^i_{\tau^i},u^i_{\tau^i}$ denote the termination time and state-action trajectory in episode $i$, and let $\theta^i$ denote the policy parameters for that episode.
Our actor-critic algorithm proceeds as follows. The critic maintains three weight vectors $w_J$, $w_M$, and $\tilde{w}_J$, and in addition maintains an estimate of $J(x_0)$, denoted by $J_0$. These parameters are updated episodically as follows:
\begin{equation}\label{eq:critic}
\begin{split}
w_J^{i+1} &= w_J^{i} + \alpha_i \sum_{t=0}^{\tau^i} \left( \sum_{s=t}^{\tau^i} r(x_s^i) - {w_J^{i-1}}^{\top} \phi_J^{\theta_i}(x_t^i,u_t^i) \right)\phi_J^{\theta_i}(x_t^i,u_t^i), \\
w_M^{i+1} &= w_M^{i} + \alpha_i \sum_{t=0}^{\tau^i} \left( \left(\sum_{s=t}^{\tau^i} r(x_s^i)\right)^2 - {w_M^{i-1}}^{\top} \phi_M^{\theta_i}(x_t^i,u_t^i) \right)\phi_M^{\theta_i}(x_t^i,u_t^i), \\
\tilde{w}_J^{i+1} &= \tilde{w}_J^{i} + \alpha_i \sum_{t=1}^{\tau^i} \left( \sum_{s=0}^{t-1}r(x_s^i)\right)\left( \sum_{s=t}^{\tau^i} r(x_s^i) - (\tilde{w}_J^{i-1})^{\top} \phi_J^{\theta_i}(x_t^i,u_t^i) \right)\phi_J^{\theta_i}(x_t^i,u_t^i), \\
J_0^{i+1} &= J_0^{i} + \alpha_i \left( \sum_{t=0}^{\tau^i} r(x_t^i) - J_0^{i-1} \right).
\end{split}
\end{equation}
The actor, in turn, updates the policy parameters according to
\begin{equation}\label{eq:actor1}
\theta^{i+1}_j = \theta^{i}_j + \beta_i \frac{\hat{\partial \eta(\theta)}}{\partial \theta},
\end{equation}
where the estimated policy gradient is given by
\begin{equation}\label{eq:actor2}
\frac{\hat{\partial \eta(\theta)}}{\partial \theta} \!=\! \sum_{t=0}^{\tau^i} \! \Dt {\log \! \pi_\theta (u_t^i|x_t^i)} \! \left( \phi_J^{\theta_i}(x_t^i,u_t^i)^{\top} \! w_J^i \!-\! \mu \! \left( \phi_M^{\theta_i}(x_t^i,u_t^i)^{\top} \! w_M^i \!-\! 2\phi_J^{\theta_i}(x_t^i,u_t^i)^{\top} \! \tilde{w}_J^i \!+\! 2J_0^i \phi_J^{\theta_i}(x_t^i,u_t^i)^{\top} \! w_J^i\right) \right).
\end{equation}

We now show that the proposed actor-critic algorithm converges w.p. 1 to a locally optimal policy. We make the following assumption on the set of locally optimal point of $\eta(\theta)$.
\begin{assumption}\label{assumption:countable_optima}
The objective function $\eta(\theta)$ has bounded second derivatives for all $\theta$. Furthermore, the set of local optima of $\eta(\theta)$ is countable.
\end{assumption}
Assumption \ref{assumption:countable_optima} is a technical requirement for the convergence of the iterates. The smoothness assumption is standard in stochastic gradient descent methods \cite{MarTsi98}. The countable set of local optima is similar to an assumption in \cite{Leslie02multiple}, and indeed our analysis follows along similar lines. When the assumption is not satisfied, our result may be extended to convergence within some set of locally optimal points. We now state our main result.

\begin{theorem}
Consider the algorithm in \eqref{eq:critic}-\eqref{eq:actor2}, and let Assumptions \ref{assumption:proper_policy}, \ref{assumption:grad_log_defined}, \ref{assumption:Phi_rank}, \ref{assumption:all_states_visited}, \ref{assumption:tilde_q_positive}, \ref{assumption:compatible_features}, and \ref{assumption:countable_optima} hold. If the step size sequences satisfy $\sum_{i} \alpha_{i} = \sum_{i} \beta_{i} = \infty $, $\sum_{i} \alpha_{i}^2 < \infty, \sum_{i} \beta_{i}^2 < \infty $, and $\lim_{i \to \infty} \frac{\beta_{i}}{\alpha_{i}} = 0$,
then almost surely
\begin{equation}
\lim_{i \to \infty} \Dt {\eta (\theta_i)} = 0.
\end{equation}
\end{theorem}

\begin{proof} (sketch)
The proof relies on representing Equations \eqref{eq:critic}-\eqref{eq:actor2} as a stochastic approximation with two time-scales \cite{Borkar97two-timescale}, where the critics parameters $w_J^i$, $w_M^i$, $\tilde{w}_J^i$, and $J_0^i$ are updated on a fast schedule while $\theta^i$ is updated on a slow schedule. Thus, $\theta^i$ may be seen as quasi-static w.r.t. $w_J^i$, $w_M^i$, $\tilde{w}_J^i$, and $J_0^i$.
We now calculate the expected updates in \eqref{eq:critic} for a given $\theta$. Note that we can safely assume that the features at the terminal state are zero, i.e., $\phi_J^{\theta}(x^*,u)=\phi_M^{\theta}(x^*,u)=0$ for all $\theta,u$. Since the reward at the terminal state is also zero by definition, we can replace the sums until $\tau^i$ with infinite sums, e.g.,
\begin{equation*}
\E^{\theta} \left[ \sum_{t=0}^{\tau} \left( \sum_{s=t}^{\tau} r(x_s) - {w_J}^{\top} \phi_J^{\theta}(x_t,u_t) \right)\phi_J^{\theta}(x_t,u_t) \right] = \E^{\theta} \left[ \sum_{t=0}^{\infty} \left( \sum_{s=t}^{\infty} r(x_s) - {w_J}^{\top} \phi_J^{\theta}(x_t,u_t) \right)\phi_J^{\theta}(x_t,u_t) \right]
\end{equation*}
By the dominated convergence theorem, using the fact that the terms in the sum are bounded and that the Markov chain is absorbing (Assumption \ref{assumption:proper_policy}), we may switch between the first sum and expectation:
\begin{equation*}
\E^{\theta} \left[ \sum_{t=0}^{\infty} \left( \sum_{s=t}^{\infty} r(x_s) - {w_J}^{\top} \phi_J^{\theta}(x_t,u_t) \right)\phi_J^{\theta}(x_t,u_t) \right] = \sum_{t=0}^{\infty} \E^{\theta} \left[ \left( \sum_{s=t}^{\infty} r(x_s) - {w_J}^{\top} \phi_J^{\theta}(x_t,u_t) \right)\phi_J^{\theta}(x_t,u_t) \right]
\end{equation*}
We now have
\begin{equation*}
\begin{split}
& \sum_{t=0}^{\infty} \E^{\theta} \left[ \left( \sum_{s=t}^{\infty} r(x_s) - {w_J}^{\top}
\phi_J^{\theta}(x_t,u_t) \right)\phi_J^{\theta}(x_t,u_t) \right] \\
&\stackrel{(a)}{=} \sum_{t=0}^{\infty} \E^{\theta} \left[ \left. \E^{\theta} \left[ \left( \sum_{s=t}^{\infty} r(x_s) - {w_J}^{\top} \phi_J^{\theta}(x_t,u_t) \right)\phi_J^{\theta}(x_t,u_t) \right| x_t,u_t \right] \right] \\
&\stackrel{(b)}{=} \sum_{t=0}^{\infty} \E^{\theta} \left[ \left( J^{\theta}(x_t,u_t) - {w_J}^{\top} \phi_J^{\theta}(x_t,u_t) \right)\phi_J^{\theta}(x_t,u_t) \right] \\
&\stackrel{(c)}{=} \sum_{t=0}^{\infty} \sum_{x\in X, u\in U} q_t^{\theta}(x,u) \left( J^{\theta}(x,u) - {w_J}^{\top} \phi_J^{\theta}(x,u) \right)\phi_J^{\theta}(x,u) \\
&\stackrel{(d)}{=} \sum_{x\in X, u\in U} q^{\theta}(x,u) \left( J^{\theta}(x,u) - {w_J}^{\top} \phi_J^{\theta}(x,u) \right)\phi_J^{\theta}(x,u) \\
&\stackrel{(e)}{=} -\nabla_{w_J} \left( \frac{1}{2} \sum_{x\in X, u\in U} q^{\theta}(x,u) \left( J^{\theta}(x,u) - {w_J}^{\top} \phi_J^{\theta}(x,u) \right)^2 \right)\\
\end{split}
\end{equation*}
where $(a)$ is by the law of iterated expectation, $(b)$ is by definition of $J^\theta(x,u)$, $(c)$ is by definition of $q_t^{\theta}(x,u)$, $(d)$ is by reordering the summations and the definition of $q^{\theta}(x,u)$, and $(e)$ follows by taking the gradient of the squared term.
Similarly, we have
\begin{equation*}
\begin{split}
& \E^{\theta} \left[ \sum_{t=0}^{\tau} \left( \left(\sum_{s=t}^{\tau} r(x_s)\right)^2 - {w_M}^{\top} \phi_M^{\theta}(x_t,u_t) \right)\phi_M^{\theta}(x_t,u_t) \right] \\
& = -\nabla_{w_M} \left( \frac{1}{2} \sum_{x\in X, u\in U} q^{\theta}(x,u) \left( M^{\theta}(x,u) - {w_M}^{\top} \phi_M^{\theta}(x,u) \right)^2 \right),
\end{split}
\end{equation*}
and
\begin{equation*}
\begin{split}
& \E^{\theta} \left[ \sum_{t=1}^{\tau} \left( \sum_{s=0}^{t-1}r(x_s)\right)\left( \sum_{s=t}^{\tau} r(x_s) - (\tilde{w}_J)^{\top} \phi_J^{\theta}(x_t,u_t) \right)\phi_J^{\theta}(x_t,u_t) \right] \\
& = -\nabla_{\tilde{w}_J} \left( \frac{1}{2} \sum_{x\in X, u\in U} \tilde{q}^{\theta}(x,u) \left( J^{\theta}(x,u) - {\tilde{w}_J}^{\top} \phi_J^{\theta}(x,u) \right)^2 \right).
\end{split}
\end{equation*}

Therefore, the updates for $w_J^i$, $w_M^i$, and $\tilde{w}_J^i$ may be associated with the following ordinary differential equations (ODE)
\begin{equation}\label{eq:ODE1}
\begin{split}
\dot{w_J} &= -\nabla_{w_J} \left( \frac{1}{2} \sum_{x\in X, u\in U} q^{\theta}(x,u) \left( J^{\theta}(x,u) - {w_J}^{\top} \phi_J^{\theta}(x,u) \right)^2 \right), \\
\dot{w_M} &= -\nabla_{w_M} \left( \frac{1}{2} \sum_{x\in X, u\in U} q^{\theta}(x,u) \left( M^{\theta}(x,u) - {w_M}^{\top} \phi_M^{\theta}(x,u) \right)^2 \right), \\
\dot{\tilde{w}}_J &= -\nabla_{\tilde{w}_J} \left( \frac{1}{2} \sum_{x\in X, u\in U} \tilde{q}^{\theta}(x,u) \left( J^{\theta}(x,u) - {\tilde{w}_J}^{\top} \phi_J^{\theta}(x,u) \right)^2 \right).
\end{split}
\end{equation}
Similarly, the update for $J_0$ is associated with the following ODE
\begin{equation}\label{eq:ODE2}
\dot{J_0} = J(x_0) - J_0.
\end{equation}
For each $\theta$, equations \eqref{eq:ODE1} and \eqref{eq:ODE2} have unique stable fixed points, denoted by $w_J^\infty$, $w_M^\infty$, $\tilde{w}_J^\infty$, and $J_0^\infty$, that satisfy
\begin{equation*}
\begin{split}
\Phi_J^\theta w_J^\infty &= \Pi^\theta_J J^\theta, \\
\Phi_M^\theta w_M^\infty &= \Pi^\theta_M M^\theta, \\
\Phi_J^\theta \tilde{w}_J^\infty &= \tilde{\Pi}^\theta_J J^\theta, \\
J_0^\infty &= J(x_0),
\end{split}
\end{equation*}
where the uniqueness of the projection weights is due to Assumption \ref{assumption:Phi_rank}.

We now return to the actor's update, Eq. \eqref{eq:actor1}-\eqref{eq:actor2}. Due to the timescale difference, $w_J^i$, $w_M^i$, $\tilde{w}_J^i$, and $J_0^i$ in the iteration for $\theta_i$ may be replaced with their stationary limit points $w_J^\infty$, $w_M^\infty$, $\tilde{w}_J^\infty$, and $J_0^\infty$, suggesting the following ODE for $\theta$
\begin{equation}\label{eq:ODE3}
\begin{split}
\dot{\theta_j} &= \left< \psi^\theta_j , \Pi^\theta_J J^\theta \right>_{q^\theta} -\mu \left( \left< \psi_j , \Pi^\theta_M M \right>_{q^\theta} + 2\left< \psi_j , \tilde{\Pi}^\theta_J J \right>_{\tilde{q}^\theta} - 2J(x_0)\left< \psi_j , \Pi^\theta_J J \right>_{q^\theta} \right) \\
&= \Dt {\eta},
\end{split}
\end{equation}
where the second equality is by Proposition \ref{prop:compatible}.

By Assumption \ref{assumption:countable_optima}, the set of stable fixed point of \eqref{eq:ODE3} is just the set of locally optimal points of the objective function $\eta(\theta)$. Let $\mathcal{Z}$ denote this set, which by Assumption \ref{assumption:countable_optima} is countable. Then, by Theorem 5 in \cite{Leslie02multiple} (which is extension of Theorem 1.1 in \cite{Borkar97two-timescale}), $\theta_i$ converges to a point in $\mathcal{Z}$ almost surely.
\end{proof}

\section{Conclusion}
We presented an actor-critic framework for a variance-penalized performance objective. Our framework extends both the policy gradient theorem and compatible features concept, which are standard tools in RL literature. To our knowledge, this is the first actor-critic algorithm that provably converges to a local optima of a variance adjusted objective function.

We remark on the practical implementation of our algorithm. The critic update equations \eqref{eq:critic} are somewhat inefficient, as they use an incremental gradient method for obtaining the least squares projections $\tilde{J}$ and $\tilde{M}$. While this is convenient for analysis purposes, more efficient approaches exist, for example, the least-squares approach proposed in \cite{Tamar2013var_TD}. Another option is to use a temporal difference (TD) approach, also proposed in \cite{Tamar2013var_TD}. While a TD approach induces bias, due to the difference between the TD fixed point and the least squares projection, this bias may be bounded, and is often small in practice. We note that a modification of these methods to produce a weighted projection is required for obtaining the weights $\tilde{w}_J$, but this could be done easily, for example by using a weighted least-squares procedure.

Finally, this work joins a collection of recent studies \cite{Tamar2012mean_var}, \cite{Tamar2013var_TD}, \cite{prasanth2013actor}, that extend RL to variance related performance criteria. The relatively simple extension of standard RL techniques to these criteria advocate their use as risk-sensitive performance measures in RL. This is in contrast to results in planning in MDPs, where global optimization of the expected return with a variance constraint was shown to be computationally hard \cite{mannor2011mean}. The algorithm considered here avoids this difficulty by considering only \emph{local} optimality.


%

\appendices
\section{Proof of Proposition \ref{prop:PG_var_theorem}\label{supp:prop:PG_var_theorem}}
\begin{proof}
Since the statement holds for each value of $\theta$ independently, we assume a fixed policy throughout the proof and drop the $\theta$ super-script and sub-script from $\pi, P, J,$ and $M$ to reduce notational clutter.
Also, let $e(i)\in \R^{| X |}$ denote a vector of zeros with the $i'$th element equal to one.

The first result is straightforward, and follows from \eqref{eq:PG_theorem2}
\begin{equation*}
\begin{split}
& \E \left[ J(x_0) \sum_{t=0}^{\infty} \Dth \log \pi (u_t|x_t) J(x_t,u_t) \right] \\
&= J(x_0)\E \left[ \sum_{t=0}^{\infty} \Dth \log \pi (u_t|x_t) J(x_t,u_t) \right] \\
&= J(x_0)\Dt {J(x_0)}.
\end{split}
\end{equation*}

We now prove the second result.
First, we have for all $x\in X$
\begin{equation*}
M(x) = \sum_{u\in U} \pi(u|x) M(x,u),
\end{equation*}
therefore, taking a gradient gives
\begin{equation}\label{eq:PG_proof_1}
\Dth M(x) = \sum_{u\in U} \Dth \pi(u|x) M(x,u) + \sum_{u\in U} \pi(u|x) \Dth M(x,u).
\end{equation}
An extension of Bellman's equation may be written for $M(x,u)$, similarly to Proposition 2 in \cite{Tamar2013var_TD}
\begin{equation*}
M(x,u) = r^2(x) + 2r(x)\sum_{y\in X}P(y|x,u)J(y) + \sum_{y\in X}P(y|x,u)M(y).
\end{equation*}
Taking a gradient of both sides gives (note that $P(y|x,u)$ is independent of the policy parameter $\theta$)
\begin{equation*}
\Dth M(x,u) =  2r(x)\sum_{y\in X}P(y|x,u) \Dth J(y) + \sum_{y\in X}P(y|x,u) \Dth M(y),
\end{equation*}
therefore, taking an expectation of $u$ given $x$
\begin{equation}\label{eq:PG_proof_2}
\sum_{u\in U} \pi(u|x) \Dth M(x,u) = 2r(x)\sum_{y\in X}P(y|x) \Dth J(y) + \sum_{y\in X}P(y|x) \Dth M(y).
\end{equation}
Plugging \eqref{eq:PG_proof_2} in \eqref{eq:PG_proof_1} and using matrix notation gives
\begin{equation*}
\Dth M = \delta M + 2RP \Dth J + P \Dth M,
\end{equation*}
where $\delta M(x) = \sum_{u\in U} \Dth \pi(u|x) M(x,u)$. Rearranging, and using the fact that $I-P$ is invertible (cf. Proposition 2 in \cite{Tamar2013var_TD}) we have
\begin{equation}\label{eq:PG_proof_3}
\Dth M = (I-P)^{-1} \left(\delta M + 2RP \Dth J\right).
\end{equation}
We now treat each of the terms in \eqref{eq:PG_proof_3} separately. For the first term, $(I-P)^{-1}\delta M$, we follow a similar procedure as in the original policy gradient theorem. We have
\begin{equation}\label{eq:PG_proof_4}
\begin{split}
e(x_0)^{\top} (I-P)^{-1} \delta M &\stackrel{(a)}{=} e(x_0)^{\top} \left( \sum_{t=0}^{\infty} P^t \right) \delta M \\
&\stackrel{(b)}{=} \sum_{t=0}^{\infty}\sum_{x\in X} P(\left. x_t = x \right)\delta M(x) \\
&\stackrel{(c)}{=} \sum_{t=0}^{\infty}\sum_{x\in X} P(\left. x_t = x \right)\sum_{u\in U} \Dth \pi(u|x) M(x,u) \\
&\stackrel{(d)}{=} \sum_{t=0}^{\infty}\sum_{x\in X,u\in U} P( x_t = x, u_t = u ) \Dth \log\pi(u|x) M(x,u) \\
&\stackrel{(e)}{=} \sum_{t=0}^{\infty} \E \left[ \Dth \log \pi_\theta (u_t|x_t) M^\theta(x_t,u_t) \right] \\
&\stackrel{(f)}{=} \E \left[ \sum_{t=0}^{\infty} \Dth \log \pi_\theta (u_t|x_t) M^\theta(x_t,u_t) \right],
\end{split}
\end{equation}
where in $(a)$ we unrolled $(I-P)^{-1}$; in $(b)$ we used a standard Markov chain property, recalling that our initial state is $x_0$; $(c)$ is by definition of $\delta M$; $(d)$ and $(e)$ are by algebraic manipulations; and $(f)$ holds by the dominated convergence theorem, using the fact that the terms in the sum are bounded (Assumption \ref{assumption:grad_log_defined}) and that the Markov chain is absorbing (Assumption \ref{assumption:proper_policy}).

Now, for the second term, using \eqref{eq:PG_theorem} we have
\begin{equation}\label{eq:PG_proof_5}
e(x_0)^{\top}(I-P)^{-1} 2RP \Dth J = e(x_0)^{\top}(I-P)^{-1} 2RP (I-P)^{-1} \delta J.
\end{equation}
We now show that the last term in Proposition \ref{prop:PG_var_theorem} is equal to the right hand side of \eqref{eq:PG_proof_5}
\begin{equation}\label{eq:PG_proof_6}
\begin{split}
& \E \left[ \sum_{t=1}^{\infty} \Dth \log \pi_\theta (u_t|x_t) J^\theta(x_t,u_t) \sum_{s=0}^{t-1} r(x_s) \right] \\
&\stackrel{(a)}{=} \E \left[ \sum_{t=0}^{\infty} r(x_t) \sum_{s=t+1}^{\infty} \Dth \log \pi_\theta (u_s|x_s) J^\theta(x_s,u_s) \right] \\
&\stackrel{(b)}{=} \sum_{t=0}^{\infty} \E \left[ r(x_t) \sum_{s=t+1}^{\infty} \Dth \log \pi_\theta (u_s|x_s) J^\theta(x_s,u_s) \right] \\
&\stackrel{(c)}{=} \sum_{t=0}^{\infty} \sum_{x\in X} P(x_t = x ) r(x) \E \left[ \left. \sum_{s=1}^{\infty} \Dth \log \pi_\theta (u_s|x_s) J^\theta(x_s,u_s) \right| x_0 = x\right] \\
&\stackrel{(d)}{=} \sum_{t=0}^{\infty} \sum_{x\in X} P(x_t = x ) r(x) e(x)^{\top} \left( \sum_{t=1}^{\infty} P^t \right) \delta J \\
&\stackrel{(e)}{=} \sum_{t=0}^{\infty} \sum_{x\in X} P(x_t = x ) e(x)^{\top} R P(I-P)^{-1} \delta J \\
&\stackrel{(f)}{=} e(x_0)^{\top}(I-P)^{-1}R P(I-P)^{-1} \delta J,
\end{split}
\end{equation}
where $(a)$ is by a change of summation order, $(c)$ is by conditioning on $x_t$ and using the law of iterated expectation, and $(d-f)$ follow a similar derivation as \eqref{eq:PG_proof_4}. Equality $(b)$ is by the dominated convergence theorem, and for it to hold we need to verify that
\begin{equation*}
\E \left[ \sum_{t=0}^{\infty} \left| r(x_t) \sum_{s=t+1}^{\infty} \Dth \log \pi_\theta (u_s|x_s) J^\theta(x_s,u_s) \right| \right] < \infty.
\end{equation*}
Let $C$ such that $|r(x)|<C$ and $\left| \Dth \log \pi_\theta (u|x) J^\theta(x,u) \right| <C$ for all $x\in X,u\in U$. By Assumption \ref{assumption:grad_log_defined} such $C$ exists. Then
\begin{equation*}
\begin{split}
\E \left[ \sum_{t=0}^{\infty} \left| r(x_t) \sum_{s=t+1}^{\infty} \Dth \log \pi_\theta (u_s|x_s) J^\theta(x_s,u_s) \right| \right] &\leq \E \left[ \sum_{t=0}^{\tau} C \sum_{s=t+1}^{\tau} C \right] \\
&\leq \E \left[ C^2 \tau^2 \right] \\
&< \infty,
\end{split}
\end{equation*}
where the first inequality holds since $x^*$ is an absorbing state, and by definition $r(x^*) = 0$ and $J(x^*,u) = 0$ for all $u$. The last inequality is a well-known property of absorbing Markov chains \cite{kemeny_finite_1960}.

Finally, multiplying \eqref{eq:PG_proof_3} by $e(x_0)^{\top}$ and using \eqref{eq:PG_proof_4}, \eqref{eq:PG_proof_5}, and \eqref{eq:PG_proof_6} gives the stated result.
\end{proof}

%
%

\ifCLASSOPTIONcaptionsoff
  \newpage
\fi



\bibliographystyle{IEEEtran}
\bibliography{IEEEabrv,./VarAC}

\begin{thebibliography}{10}
\providecommand{\url}[1]{#1}
\csname url@samestyle\endcsname
\providecommand{\newblock}{\relax}
\providecommand{\bibinfo}[2]{#2}
\providecommand{\BIBentrySTDinterwordspacing}{\spaceskip=0pt\relax}
\providecommand{\BIBentryALTinterwordstretchfactor}{4}
\providecommand{\BIBentryALTinterwordspacing}{\spaceskip=\fontdimen2\font plus
\BIBentryALTinterwordstretchfactor\fontdimen3\font minus
  \fontdimen4\font\relax}
\providecommand{\BIBforeignlanguage}[2]{{%
\expandafter\ifx\csname l@#1\endcsname\relax
\typeout{** WARNING: IEEEtran.bst: No hyphenation pattern has been}%
\typeout{** loaded for the language `#1'. Using the pattern for}%
\typeout{** the default language instead.}%
\else
\language=\csname l@#1\endcsname
\fi
#2}}
\providecommand{\BIBdecl}{\relax}
\BIBdecl

\bibitem{BT96}
D.~P. Bertsekas and J.~N. Tsitsiklis, \emph{Neuro-Dynamic Programming}.\hskip
  1em plus 0.5em minus 0.4em\relax Athena Scientific, 1996.

\bibitem{Puterman1994}
M.~L. Puterman, \emph{Markov decision processes: discrete stochastic dynamic
  programming}.\hskip 1em plus 0.5em minus 0.4em\relax John Wiley \& Sons,
  Inc., 1994.

\bibitem{Konda2002actor}
V.~Konda, ``Actor-critic algorithms,'' Ph.D. dissertation, Dept. Comput. Sci.
  Elect. Eng., MIT, Cambridge, MA, 2002.

\bibitem{sutton_policy_2000}
R.~S. Sutton, D.~{McAllester}, S.~Singh, and Y.~Mansour, ``Policy gradient
  methods for reinforcement learning with function approximation,''
  \emph{Advances in neural information processing systems}, vol.~12, no.~22,
  2000.

\bibitem{peters2008natural}
J.~Peters and S.~Schaal, ``Natural actor-critic,'' \emph{Neurocomputing},
  vol.~71, no.~7, pp. 1180--1190, 2008.

\bibitem{grondman2012survey}
I.~Grondman, L.~Busoniu, G.~A. Lopes, and R.~Babuska, ``A survey of
  actor-critic reinforcement learning: Standard and natural policy gradients,''
  \emph{IEEE Transactions on Systems, Man, and Cybernetics, Part C:
  Applications and Reviews}, vol.~42, no.~6, pp. 1291--1307, 2012.

\bibitem{Tamar2012mean_var}
A.~Tamar, D.~Di~Castro, and S.~Mannor, ``Policy gradients with variance related
  risk criteria,'' in \emph{International conference on machine learning},
  2012.

\bibitem{sato2001td}
M.~Sato, H.~Kimura, and S.~Kobayashi, ``{TD} algorithm for the variance of
  return and mean-variance reinforcement learning,'' \emph{Transactions of the
  Japanese Society for Artificial Intelligence}, vol.~16, pp. 353--362, 2001.

\bibitem{prasanth2013actor}
\BIBentryALTinterwordspacing
P.~L.A. and M.~Ghavamzadeh, ``{Actor-Critic Algorithms for Risk-Sensitive
  MDPs},'' Technical Report, 2013. [Online]. Available:
  \url{http://hal.inria.fr/hal-00794721}
\BIBentrySTDinterwordspacing

\bibitem{spall1992multivariate}
J.~C. Spall, ``Multivariate stochastic approximation using a simultaneous
  perturbation gradient approximation,'' \emph{IEEE Transactions on Automatic
  Control}, vol.~37, no.~3, pp. 332--341, 1992.

\bibitem{Tamar2013var_TD}
A.~Tamar, D.~Di~Castro, and S.~Mannor, ``Temporal difference methods for the
  variance of the reward to go,'' \emph{JMLR W{\&}CP}, vol.~28, no.~3, pp.
  495--503, 2013.

\bibitem{Ber2012DynamicProgramming}
D.~P. Bertsekas, \emph{Dynamic Programming and Optimal Control, Vol II},
  4th~ed.\hskip 1em plus 0.5em minus 0.4em\relax Athena Scientific, 2012.

\bibitem{MarTsi98}
P.~Marbach and J.~N. Tsitsiklis, ``Simulation-based optimization of {M}arkov
  reward processes,'' \emph{IEEE Transactions on Automatic Control}, vol.~46,
  no.~2, pp. 191--209, 1998.

\bibitem{Leslie02multiple}
D.~S. Leslie and E.~Collins, ``Convergent multiple-timescales reinforcement
  learning algorithms in normal form games,'' \emph{Annals of Applied
  Probability}, vol.~13, pp. 1231--1251, 2002.

\bibitem{Borkar97two-timescale}
V.~S. Borkar, ``Stochastic approximation with two time scales,'' \emph{Systems
  \& Control Letters}, vol.~29, no.~5, pp. 291 -- 294, 1997.

\bibitem{mannor2011mean}
S.~Mannor and J.~N. Tsitsiklis, ``Algorithmic aspects of mean-variance
  optimization in {M}arkov decision processes,'' \emph{European Journal of
  Operational Research}, vol. 231, no.~3, pp. 645 -- 653, 2013.

\bibitem{kemeny_finite_1960}
J.~G. Kemeny and J.~L. Snell, \emph{Finite markov chains}.\hskip 1em plus 0.5em
  minus 0.4em\relax Van Nostrand, 1960.

\end{thebibliography}
\end{document}